\documentclass[conference]{IEEEtran}
\IEEEoverridecommandlockouts
\usepackage{cite}
\usepackage{amsmath,amssymb,amsfonts}
\usepackage{graphicx}
\usepackage{textcomp}
\usepackage{xcolor}
\usepackage{amsthm}
\usepackage{graphicx}
\usepackage{url}
\usepackage{times}
\usepackage{comment}
\usepackage{color}
\usepackage{cite}
\usepackage{mathtools}
\usepackage{algorithm}
\usepackage{algpseudocode}
\usepackage{multirow}
\usepackage{subcaption}
\usepackage{booktabs}
\usepackage{enumitem}

\newtheorem{definition}{Definition}

\newtheorem{lemma}{Lemma}
\newtheorem{remark}{Remark}

\newcommand{\vect}[1]{\begin{bmatrix}#1\end{bmatrix}}
\newcommand{\conftr}{\textsc{ConfTr}}
\newcommand{\tlicp}{\textsc{TLICP}}

\def\G{\square}
\def\F{\lozenge}

\makeatletter % changes the catcode of @ to 11
\newcommand{\linebreakand}{%
  \end{@IEEEauthorhalign}
  \hfill\mbox{}\par
  \mbox{}\hfill\begin{@IEEEauthorhalign}
}
\makeatother % changes the catcode of @ back to 12

\theoremstyle{definition}

\def\BibTeX{{\rm B\kern-.05em{\sc i\kern-.025em b}\kern-.08em
    T\kern-.1667em\lower.7ex\hbox{E}\kern-.125emX}}
    
\begin{document}

\title{Conformal Prediction for Signal Temporal Logic Inference
\thanks{DISTRIBUTION STATEMENT A. Approved for public release. Distribution is unlimited. This material is based upon work supported by the Department of the Army under Air Force Contract No. FA8702-15-D-0001 or FA8702-25-D-B002. Any opinions, findings, conclusions or recommendations expressed in this material are those of the author(s) and do not necessarily reflect the views of the Department of the Army. © 2025 Massachusetts Institute of Technology. Delivered to the U.S. Government with Unlimited Rights, as defined in DFARS Part 252.227-7013 or 7014 (Feb 2014). Notwithstanding any copyright notice, U.S. Government rights in this work are defined by DFARS 252.227-7013 or DFARS 252.227-7014 as detailed above. Use of this work other than as specifically authorized by the U.S. Government may violate any copyrights that exist in this work.}
}

\author{\IEEEauthorblockN{1\textsuperscript{st} Danyang Li}
\IEEEauthorblockA{\textit{Mechanical Engineering Department} \\
\textit{Boston University}\\
Boston, MA, USA\\
danyangl@bu.edu}
\and
\IEEEauthorblockN{2\textsuperscript{nd} Yixuan Wang}
\IEEEauthorblockA{\textit{Department of Mechanical Engineering} \\
\textit{University of California, Riverside}\\
Riverside, CA, USA\\
ywang1457@ucr.edu}
\and
\IEEEauthorblockN{3\textsuperscript{rd} Matthew Cleaveland}
\IEEEauthorblockA{\textit{Lincoln Laboratory} \\
\textit{Massachusetts Institute of Technology}\\
Boston, MA, USA\\
Matthew.Cleaveland@ll.mit.edu}
\and

\linebreakand

\IEEEauthorblockN{4\textsuperscript{th} Mingyu Cai}
\IEEEauthorblockA{\textit{Department of Mechanical Engineering} \\
\textit{University of California, Riverside}\\
Riverside, CA, USA\\
mingyuc@ucr.edu}
\and
\IEEEauthorblockN{5\textsuperscript{th} Roberto Tron}
\IEEEauthorblockA{\textit{Mechanical Engineering Department} \\
\textit{Boston University}\\
Boston, MA, USA\\
tron@bu.edu}
}

\maketitle

%%%%%%%%%%%%%%%%%%%%%%%%%%%%%%%%%%%%%%%%%%%%%%%%%%%%%%%%%%%%%%%%%%%%%%%%%%%%%%%%
\begin{abstract}
Signal Temporal Logic (STL) inference seeks to extract human-interpretable rules from time-series data, but existing methods lack formal confidence guarantees for the inferred rules. Conformal prediction (CP) is a technique that can provide statistical correctness guarantees, but is typically applied as a post-training wrapper without improving model learning.
Instead, we introduce an end-to-end differentiable CP framework for STL inference that enhances both reliability and interpretability of the resulting formulas.
We introduce a robustness-based nonconformity score, embed a smooth CP layer directly into training, and employ a new loss function that simultaneously optimizes inference accuracy and CP prediction sets with a single term. Following training, an exact CP procedure delivers statistical guarantees for the learned STL formulas. Experiments on benchmark time-series tasks show that our approach reduces uncertainty in predictions (i.e., it achieves high coverage while reducing prediction set size), and improves accuracy (i.e., the number of misclassifications when using a fixed threshold) over state-of-the-art baselines.
\end{abstract}

%%%%%%%%%%%%%%%%%%%%%%%%%%%%%%%%%%%%%%%%%%%%%%%%%%%%%%%%%%%%%%%%%%%%%%%%%%%%%%%%
\section{Introduction}
Motivated by the growing demand for transparency, trustworthiness, and human interpretability, inherently interpretable machine learning models have gained increasing attention, particularly in safety-critical domains such as autonomous driving and healthcare~\cite{8718798, bartocci2014data}. Signal Temporal Logic (STL)~\cite{maler2004monitoring, fainekos2009robustness} has emerged as a compelling framework for this purpose, providing a formal language capable of expressing rich behavioral properties of dynamical systems. STL inference~\cite{baharisangari2021uncertainty}, typically framed as a classification task, learns symbolic formulas from time-series data, enabling rigorous reasoning in a form comprehensible to humans. Early inference approaches typically relied on combinatorial optimization or symbolic search to derive specifications from labeled trajectories~\cite{8550271, jha2019telex, bartocci2014data}. Bombara and Belta~\cite{8550271} proposed an online algorithm for learning STL classifiers from signal traces. Jha et al.~\cite{jha2019telex} presented TeLEx, a framework leveraging tightness metrics to synthesize STL formulas from positive examples alone. While effective in restricted settings, these methods struggled with scalability and expressiveness as data and formulas grow more complex. Recently, neural-symbolic frameworks have been proposed to overcome these limitations by embedding STL inference into differentiable neural architectures~\cite{fronda2022differentiable, li2024tlinet}. In particular, our prior work introduced TLINet~\cite{li2024tlinet}, a differentiable STL inference network that jointly learns formula structure and parameters end-to-end, significantly improving scalability and expressiveness. However, no standard criterion exists for assessing the quality of an inferred STL formula. To address this gap, we leverage conformal prediction~\cite{vovk2005algorithmic}.

Conformal prediction (CP) provides statistical guarantees by constructing prediction sets that contain the true label with a user-specified probability, known as \emph{coverage}~\cite{vovk2005algorithmic, angelopoulos2021gentle, 004051271e0b4944884b98180eb898c1, stutz2023conformal}. It uses \emph{nonconformity scores} to measure the agreement between a sample and the model (a large value indicates poor agreement) and calibrates uncertainty through the size of the prediction set, called \emph{inefficiency}, where larger sets indicate higher uncertainty. CP only assumes that samples are exchangeable and drawn from the same distribution, so it imposes no further restrictions on the data distribution or on the types of models. These properties make CP a natural tool for assessing the quality of inference models. 

There are various CP methods such as classical CP~\cite{shafer2008tutorial} and inductive CP~\cite{papadopoulos2002inductive, lei2017distributionfreepredictiveinferenceregression}. Classical CP requires retraining the model for each sample, which is computationally expensive. Inductive CP avoids this by splitting data into training, calibration, and testing sets: training the model once, then using calibration set to construct prediction sets for testing set, thereby avoiding repeated training. These CP methods are commonly applied after training as a calibration step, so it neither influences the learning process nor improves the model~\cite{angelopoulos2021gentle}. To address this limitation, Stutz et al.~\cite{stutz2021learning} proposed conformal training (\conftr), which incorporates CP directly into training to jointly optimize accuracy and inefficiency, yielding smaller prediction sets with improved coverage~\cite{einbinder2022training}. However, its application to the inference models requires two significant considerations. First, traditional classifiers generally output class probabilities, so nonconformity scores can be obtained simply by taking the difference between the predicted probability and the ground-truth label. STL inference frameworks, however, produce real-valued robustness scores indicating satisfaction or violation of the learned formula. Since there is no ground-truth robustness and no natural way to map robustness values to class probabilities, standard error‑based nonconformity measures are not sufficient. Second, \conftr\ explicitly constructs prediction sets during training and focuses on minimizing their size. This objective ignores the possibility of producing an empty set, whereas the desired set is of size one.

Despite the significant progress of STL inference and conformal prediction, their integration remains relatively unexplored. Existing work that combines STL with CP targets trajectory prediction rather than inference. For example, Lindemann et al.~\cite{lindemann2023conformal} construct prediction regions for future system states by evaluating the robustness of a pre‑defined STL specification. Cleaveland et al.~\cite{cleaveland2024conformal} extend CP to long‐horizon forecasting of time‑series data. Soroka et al.~\cite{soroka2024learning} focus on inferring STL predicates with a confidence interval derived from conformalized quantile regression, but their method is restricted to simple predicates which lack expressivity, and provides statistical guarantees only for future trajectories rather than for the satisfaction or violation of a learned formula on observed data.

\textbf{\textit{Contributions}} Our contributions can be summarized as follows: \textit{(i)} We design a new nonconformity score that operates directly on robustness values, even though ground‑truth robustness is unavailable, and develop smooth approximations that make this score suitable for training. \textit{(ii)} We extend \conftr\ based on \textit{(i)} to handle STL inference models and provide a comparison with our proposed method. \textit{(iii)} Leveraging inductive CP, we develop a smooth approximation of the CP procedure without explicit construction of prediction sets, allowing joint optimization of the inference model and its conformal wrapper; after training, non‑smoothed CP algorithm is applied to provide rigorous statistical guarantee on unseen data. \textit{(iv)} We propose a novel loss function that combines the classification error and CP goal in a single term, unlike \conftr, which relies on a separate regularization term. Our formulation requires no additional tuning and naturally avoids empty prediction sets. \textit{(v)} In experiments, we show that our method improves both efficiency and coverage compared with state‑of‑the‑art conformal training baselines.

%%%%%%%%%%%%%%%%%%%%%%%%%%%%%%%%%%%%%%%%%%%%%%%%%%%%%%%%%%%%%%%%%%%%%%%%%%%%%%%%
\section{Preliminaries}
\subsection{Signal Temporal Logic}\label{sec:STL}
Signal Temporal Logic (STL)~\cite{maler2004monitoring, fainekos2009robustness} is an expressive formal language that can describe spatial-temporal characteristics of time-series data. Let $X=\vect{x_0, x_1,\cdots,x_T}$ be a discrete, finite signal, where $x_t\in \mathbb{R}^d$ is the state of $X$ at time $t$. The syntax of STL is recursively defined as:
\begin{equation}\label{eq:stl}
    \phi \Coloneqq \top \mid \mu \mid \neg \phi \mid \phi_1\land\phi_2 \mid \phi_1\lor\phi_2 \mid \F_I\phi \mid \G_I\phi \mid \phi_1 U_I \phi_2,
\end{equation}
where $\phi_1$ and $\phi_2$ are STL formulas, $\mu$ is a predicate, $\mu:=a^\top x_t\geq b$, $\neg$(\emph{Negation}), $\land$ (\emph{And}) and $\lor$ (\emph{Or}) are Boolean operators, $\F$ (\emph{Eventually}), $\G$ (\emph{Always}) and $U$(\emph{Until}) are temporal operators, $I$ represents the time interval.

The quantitative semantics of STL is the robustness $\rho^{\phi}(X,t)\in\mathbb{R}$ that indicates how robustly $\phi$ is satisfied or violated\cite{donze2010robust,fainekos2009robustness}. If $\rho^{\phi}(X,t)> 0$, it holds that
$X$ satisfies $\phi$ at time $t$, denoted as $(X,t)\models\phi$. The larger $\rho^{\phi}(X,t)$ is, the more robustly $\phi$ is satisfied.

\textbf{\textit{STL Inference}} Signal Temporal Logic Inference aims to learn an STL formula $\phi$ from a labeled dataset $\mathcal{S} := \{(X^{(1)},Y^{(1)}),\cdots,(X^{(N)},Y^{(N)})\}$ with the label $Y^{(i)}\in\{1,-1\}$. Here $Y^{(i)}=1$ indicates that the signal $X^{(i)}$ exhibits a desired behavior, whereas $Y^{(i)}=-1$ indicates the opposite. During training, $\phi$ is treated as a classifier: for each $(X^{(i)},t)$, we predict $\tilde{Y}^{(i)}=1$ if $\rho^{\phi}(X^{(i)},t)> 0$, otherwise we predict $\tilde{Y}^{(i)}=-1$. Equivalently, $(X^{(i)},t)\models \phi \iff \tilde{Y}^{(i)}=1$.

\subsection{Conformal Prediction}\label{sec:CP}
Conformal Prediction (CP)~\cite{vovk2005algorithmic,shafer2008tutorial} is a framework for generating prediction sets for any model with statistical guarantees. In our work, we focus on a binary classification task with label set $\mathcal{K} = \{1, -1\}$; we adopt \emph{Inductive CP} (ICP, \cite{004051271e0b4944884b98180eb898c1}), in which the dataset $\mathcal{S}$ is partitioned into a training set $\mathcal{S}_{\text{train}}$, a calibration set $\mathcal{S}_{\text{cal}}$ and a test set $\mathcal{S}_{\text{test}}$. A classifier $\pi$ is first trained on $\mathcal{S}_{\text{train}}$. The calibration set $\mathcal{S}_{\text{cal}}$ is then used to evaluate the performance on the unseen test set $\mathcal{S}_{\text{test}}$. Specifically, given a trained classifier $\pi$ and a calibration set $\mathcal{S}_{\text{cal}} = \{(X_i, Y_i)\}_{i=1,...,n}$, CP constructs a prediction set $C(X) \subseteq \mathcal{K}$ for a test input $(X,Y) \in \mathcal{S}_{\text{test}}$ such that
\begin{equation}
    P(Y\in C(X))\geq 1-\alpha,
\end{equation}
where $(X,Y)$ and calibration data are drawn exchangeably from the same distribution, and $\alpha \in [0,1]$ is a user-specified confidence level. This coverage guarantee ensures that the true label lies within the prediction set (i.e., it is covered) with probability at least $1 - \alpha$. 

% The size of the prediction set, denoted as $|C(X)|$, is often referred to as the model inefficiency and serves as a measure of model uncertainty: larger sets indicate higher uncertainty about the prediction.

In our work, we adopt ICP in two formulations. Both utilize the same real-valued nonconformity score $E_{\pi}(X, k)$ computed from the trained classifier $\pi$, but differ in how quantiles are parametrized. In the first formulation, the prediction set is constructed using a threshold:
\begin{equation}\label{eq:C}
    C(X) = \{k\in\mathcal{K}: E_{\pi}(X,k)\leq \tau\},
\end{equation}
where $\tau$ is the $\lceil(1-\alpha)(n+1)\rceil/n$ quantile of the nonconformity scores on the calibration set, i.e., $\{E_{\pi}(X_i,Y_i)\}_{i=1,...,n}$. The second formulation constructs the prediction set based on a $p$-value for each potential label $k \in \mathcal{K}$. The $p$-value is defined as:
\begin{equation}\label{eq:p}
    p_k = \frac{|\{i\in\{1,...,n\}:E_{\pi}(X_i,Y_i)\geq E_{\pi}(X,k)\}|+1}{n+1}.
\end{equation}
The prediction set is then given by:
\begin{equation}\label{eq:cp pvalue}
    C(X)=\{k\in\mathcal{K}:p_k>\alpha\}.
\end{equation}
Theoretically, both formulations are equivalent~\cite{stutz2023conformal}; for a fixed $\alpha$ they will yield identical prediction sets.

%%%%%%%%%%%%%%%%%%%%%%%%%%%%%%%%%%%%%%%%%%%%%%%%%%%%%%%%%%%%%%%%%%%%%%%%%%%%%%%%
\section{Problem Statement}
Given a dataset $\mathcal{S}$, our goal is to learn the optimal parameters $\theta$ of an STL inference classifier $\pi_{\theta}$ such that:
\begin{enumerate}
    \item The inferred STL formula $\phi_{\theta}$ minimizes the misclassification rate (MCR), defined as:
    \begin{equation}
    \begin{aligned}
        \text{MCR} = &\frac{1}{N}(|X^{(i)} \models \phi_{\theta} \land Y^{(i)} = -1| \\
        &+ |X^{(i)} \not\models \phi_{\theta} \land Y^{(i)} = 1|).
    \end{aligned}
    \end{equation}
    \item The resulting prediction sets provide a valid coverage guarantee, i.e., given a confidence level $\alpha$, for all test samples $(X,Y)$,
    \begin{equation}
        P(Y\in C_{\theta}(X))\geq 1-\alpha,
    \end{equation}
    where $C_{\theta}(X)$ is the prediction set induced by $\pi_{\theta}$ using conformal prediction after training.
    \item The uncertainty, measured by the inefficiency, is minimized across test inputs.
\end{enumerate}

% {\color{red} What's the inefficiency? Any definition regarding p-value?}

%%%%%%%%%%%%%%%%%%%%%%%%%%%%%%%%%%%%%%%%%%%%%%%%%%%%%%%%%%%%%%%%%%%%%%%%%%%%%%%%
\section{Nonconformity Score for STL Inference}
STL inference models differ from conventional classifiers in that they output a real-valued robustness score indicating how strongly a signal satisfies a learned formula. This quantity is not, however, an immediate candidate for a non-conformity score. The challenge is that there is no ground-truth robustness for each sample, only the sign is important; as a result, a small robustness value does not always signify higher predictive uncertainty (it depends on the distribution of robustness values for other data points). To overcome this difficulties, we introduce a dedicated nonconformity score specifically designed for robustness-based STL inference.

\subsection{STL margin}
Before defining the nonconformity score, we review the concept of the \emph{margin} for STL inference \cite{li2024multi}, which is inspired by the notion of margin in Support Vector Machines (SVMs)~\cite{cortes1995support}. In SVM, a hyperplane separates data points from different classes. The margin refers to the distance between this hyperplane and the closest data points from each class, and the goal is to maximize the margin.

For STL inference problem, we treat the inferred STL formula $\phi_{\theta}$ as a mapping from input signals to a point, i.e., the robustness. From Section \ref{sec:STL}, we consider the origin as the separating hyperplane, and the robustness can be interpreted as a signed distance to this decision boundary. %The problem thus reduces to finding an STL formula $\phi_{\theta}$ that maps inputs to the correct side.
We formalize the STL margin as follows:
\begin{definition}[Binary-class STL margin]\label{def:STL margin}
Given a labeled dataset $\mathcal{S}$ and the learned classifier $\pi_{\theta}$, the margin of the inferred formula $\phi_{\theta}$ respect to $\mathcal{S}$ is defined as:
\begin{equation}\label{eq:margin}
\begin{aligned}
    m_{\pi_{\theta},\mathcal{S}} = \min_{i} \{\text{ReLU}(\pi_{\theta}(X^{(i)})\cdot Y^{(i)})\},\forall i\in\{1,\ldots,N\},
\end{aligned}
\end{equation}
where $\text{ReLU}(x)=\max (0,x)$.
\end{definition}

Intuitively, Section \ref{eq:margin} defines the margin as the minimum distance from the origin to the robustness of correctly classified data points; this definition treats the distances for the positive and negative classes in a symmetric way. Specifically, as illustrated in Figure~\ref{fig:1dmargin}, the margin is given by $m_{\pi_{\theta},\mathcal{S}}=\min\{m_p,m_n\}$. Thus, maximizing the margin helps to balance the decision boundary, preventing it from being biased toward one class, and thereby achieving a better separation between the two classes.

\begin{figure}[ht]
    \centering
    \includegraphics[width=1.0\linewidth]{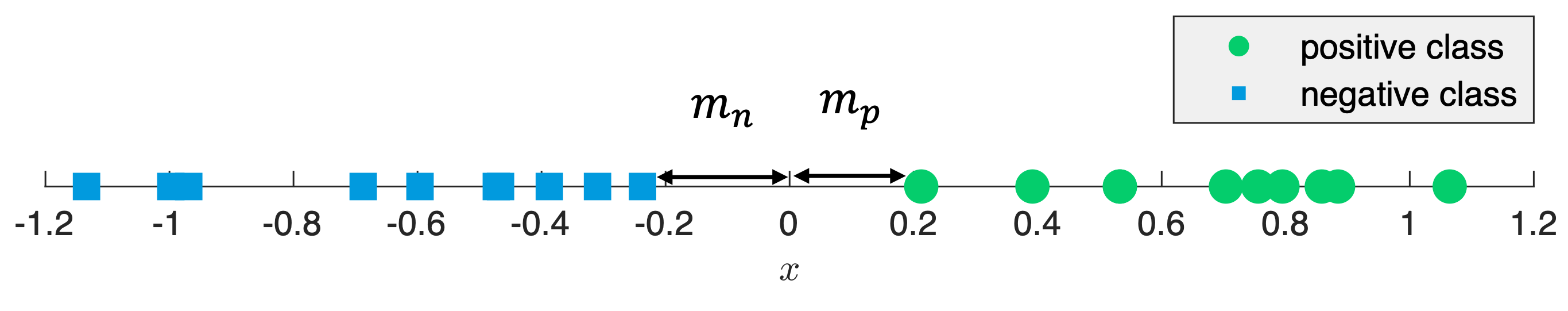}
    \caption{An example of binary-class STL margin.}
    \label{fig:1dmargin}
\end{figure}

\subsection{Nonconformity Score}\label{sec:score}
In this section, we propose a nonconformity score leveraging the STL margin.
To compute the margin, we follow the strategy used in \cite{cleaveland2024conformal}, and divide the calibration set $\mathcal{S}_{\text{cal}}$ into two subsets $\mathcal{S}_{\text{cal1}},\mathcal{S}_{\text{cal2}}$. The first set $\mathcal{S}_{\text{cal1}}$ is used to compute the margin $m_{\pi_{\theta},\mathcal{S}_{\text{cal1}}}$. For brevity, we use $m$ to denote $m_{\pi_{\theta},\mathcal{S}_{\text{cal1}}}$ in the following paragraph. Using this margin, we define the nonconformity score for each calibration sample $(X_i,Y_i)\in\mathcal{S}_{\text{cal2}}$ as:
\begin{equation}\label{eq:score}
    E_{\pi}(X_i,Y_i) = \begin{cases}
    M,  &\textrm{if } \pi(X_i)>m,Y_i=-1,\\
    M,  &\textrm{if } \pi(X_i)<-m,Y_i=1,\\
    1,  &\textrm{if } -m<\pi(X_i)\leq m,Y_i=-1,\\
    1,  &\textrm{if } -m\leq\pi(X_i)< m ,Y_i=1,\\
    0,  &\textrm{if } \pi(X_i)\leq -m,Y_i=-1,\\
    0,  &\textrm{if } \pi(X_i)\geq m,Y_i=1,\\
    \end{cases}
\end{equation}
where $M\in(1,\infty]$ is a user-defined large value. Note that the exact value of $M$ does not affect the resulting prediction set, only the ordering $0<1<M$ matters. This formulation is inspired by the nonconformity score used for SVMs~\cite{shafer2008tutorial}.

The intuition is as follows: First, because the ground-truth robustness is unknown, we use the learned margin $m$ to measure the separation and assign piecewise-constant scores by region: correct side outside the margin (lowest nonconformity), within the margin (intermediate), and wrong side outside the margin (highest). Second, since robustness magnitudes depend on measurement units, the score should be consistent under uniform rescaling; only the sign of $\pi_{\theta}(X)$ and whether $|\pi_{\theta}(X)|$ exceeds $m$ should matter. For example, in the predicates of \eqref{eq:stl}, scaling $a$ and $b$ by any positive constant preserves the sign of the robustness and thus the classification, so the score must reflect this invariant.

\begin{lemma}\label{lemma:invariant}
The proposed nonconformity scores are invariant to changes of measurement units (equivalently, to positive rescalings of the robustness).
\end{lemma}
\begin{proof}
  Let $c>0$ represent a change of scale (i.e., a change of unit) from the robustness $\pi_{\theta}(X)$ to the robustness $\pi'_{\theta}(X)=c\pi_{\theta}(X)$. From Section \ref{def:STL margin}, the corresponding margin is then $m'=cm$. Multiplication by a positive constant preserves all inequalities in \eqref{eq:score}, so $\pi_{\theta}(X)\gtrless\pm m$ if and only if $\pi'_{\theta}(X)\gtrless\pm m'$; as a result, the assigned score is unchanged, i.e., $E_{\pi'}(X,Y) = E_{\pi}(X,Y)$.
\end{proof}

%%%%%%%%%%%%%%%%%%%%%%%%%%%%%%%%%%%%%%%%%%%%%%%%%%%%%%%%%%%%%%%%%%%%%%%%%%%%%%%%
\section{Differentiable Conformal Predictor for STL Inference}\label{sec:diff cp}
Many existing STL inference frameworks learn to predict specification satisfaction during training, and then apply CP post-hoc to obtain coverage guarantees. However, these frameworks are not explicitly informed about the conformalization procedure during training. We aim to improve CP performance by incorporating it directly into the training process, such that the learned inference classifier $\pi_{\theta}$ is optimized not only for classification accuracy but also for reduced inefficiency. To this end, we propose a differentiable Temporal Logic Inference Conformal Predictor (\tlicp), adopting the concept of conformal training from \cite{stutz2021learning}, in which a differentiable conformal predictor is trained jointly with the model. Unlike prior work, \tlicp\ introduces a novel differentiable prediction process and a novel loss function. Training such a conformal predictor requires both differentiable nonconformity score and prediction process with respect to the model parameters $\theta$.

\subsection{Differentiable Nonconformity Score}
The nonconformity score introduced in \eqref{eq:score} is not differentiable and therefore cannot be directly used in the training process. To address this, we propose a differentiable approximation of the nonconformity score in \eqref{eq:diff score} using the sigmoid function $\sigma(x)=1/(1+e^{-x})$ as:
\begin{equation}\label{eq:diff score}
\begin{aligned}
    \tilde{E}_{\theta}(X_i,Y_i) = &f_1(\pi_{\theta}(X_i)Y_i)f_2(\pi_{\theta}(X_i)Y_i) + Mf_3(\pi_{\theta}(X_i)Y_i),
\end{aligned}
\end{equation}
where
\begin{equation}
\begin{aligned}
    &f_1(x) = \sigma(-(x-m)/T_1),\\
    &f_2(x) = \sigma((x+m)/T_2),\\
    &f_3(x) = \sigma(-(x+m)/T_3),\\
\end{aligned}
\end{equation}
and $\{T_i\}_{i=1,2,3}$ are temperature hyperparameters that control the steepness of the curve. The first term $f_1\cdot f_2$ smoothly approximates the indicator function for $\pi_{\theta}(X_i)Y_i\in [-m,m]$, assigning a score close to $1$ within the margin and near $0$ otherwise. The second term $M\cdot f_3$ is close to $M$ if $\pi_{\theta}(X_i)Y_i>m$ and close to 0 otherwise. Figure \ref{fig:score} shows an example of the approximated nonconformity score $\tilde{E}_{\theta}(X,Y)$ with respect to $\pi_{\theta}(X)$ and $Y$.

\begin{figure}
    \centering
    \includegraphics[width=1\linewidth]{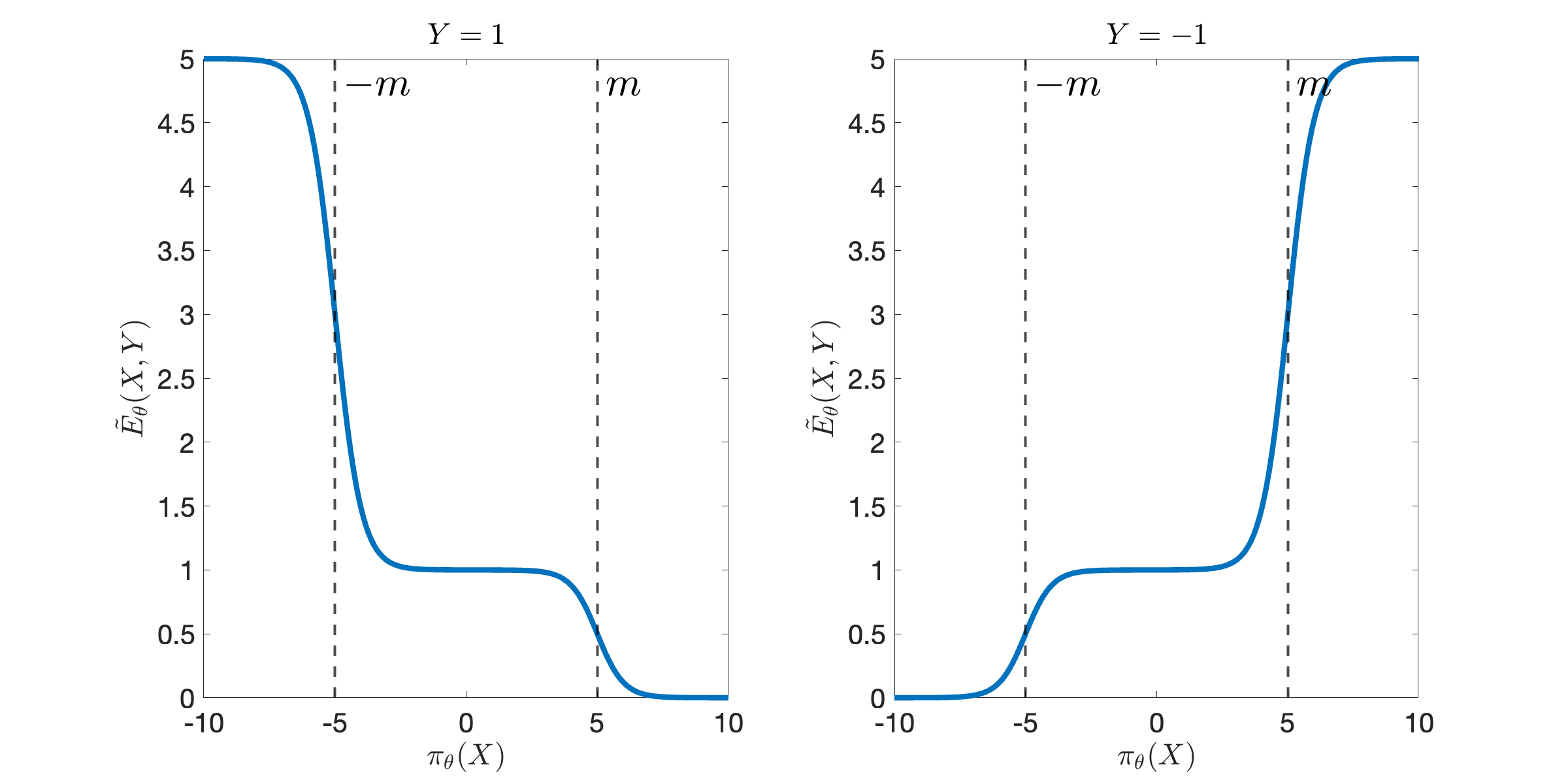}
    \caption{The approximated nonconformity score $\tilde{E}_{\theta}(X,Y)$ with $m=5$, $M=5$, $T_1 = T_2 = T_3 = 0.5$.}
    \label{fig:score}
\end{figure}

\subsection{\conftr{} using differentiable thresholds}
Stutz et al.~\cite{stutz2021learning} has introduced \conftr, a conformal-training scheme that builds prediction sets from a differentiable threshold on nonconformity scores. We briefly outline its calibration and prediction procedures below.

\textbf{\textit{Calibration}} A smooth quantile approximation, implemented with differentiable sorting techniques~\cite{blondel2020fastdifferentiablesortingranking,cuturi2019differentiable} computes an approximated threshold $\tilde{\tau}$ being differentiable w.r.t. the model’s parameters $\theta$. This step is denoted as $\tilde{\tau} = \text{SMOOTHCAL}(\{\pi_{\theta}(X_i),Y_i)\}_{i=1,...,n},\alpha)$.

\textbf{\textit{Prediction}} Given a test point $X$,  the soft inclusion of a candidate label $k$ in the prediction set is
\begin{equation}
  \label{eq:C smooth}
  C_{\theta,k}(X;\tilde{\tau})=\sigma((\tilde{\tau}-\tilde{E}_{\theta}(X,k))/T_c)\in [0,1],
\end{equation}
where $T_c$ is a temperature hyperparameter. This operation, which is the smooth counterpart of \eqref{eq:C}, is denoted as $C_{\theta,k}(X;\tilde{\tau}) = \text{SMOOTHPRED}(\pi_{\theta}(X),\tilde{\tau})$.

We extend \conftr{} to inference models by incorporating the differentiable nonconformity score from Section \ref{eq:diff score}, thereby extending conformal training to the temporal-logic setting.

\subsection{\tlicp{} using $p$-values}

Although \conftr{} introduces differentiability into conformal prediction, its two-stage design requires explicit thresholding and prediction set construction, adding complexity to both training and optimization.
Our proposed method, \tlicp, relies solely on $p$-values, eliminating separate calibration and prediction stages.
For a test sample $X$ and candidate label $k$, the soft $p$-value is
\begin{equation}
    p_{\theta}(X;k) = \frac{\sum_{i=1}^n \sigma\big((\tilde{E}_{\theta}(X,k)-\tilde{E}_{\theta}(X_i,Y_i))/T_p\big)+1}{n+1},
\end{equation}
where $T_p$ is a temperature hyperparameter. This operation is the smooth counterpart of \eqref{eq:p}, and is denoted by $p_{\theta}(X;k) = \text{DIFFP}(\tilde{E}_{\theta}(X,k),\{\tilde{E}_{\theta}(X_i,Y_i)\}_{i=1,...,n})$.
By collapsing calibration and prediction into a single step, \tlicp{} simplifies the training process and loss formulation, and improves computational efficiency. The next section details how these soft $p$-values are incorporated into our overall optimization.

%%%%%%%%%%%%%%%%%%%%%%%%%%%%%%%%%%%%%%%%%%%%%%%%%%%%%%%%%%%%%%%%%%%%%%%%%%%%%%%%
\section{Learning Conformal Prediction for STL Inference}
The goal of training a conformal predictor is to allow gradients to flow through the entire prediction procedure. To this end, each training batch $B\subset\mathcal{S}_{\text{train}}$ is split into two equal parts: a calibration subset $B_{\text{cal}}$ and a test subset $B_{\text{test}}$, such that $B = B_{\text{cal}}\uplus B_{\text{test}}$. The classifier $\pi$ is updated on the full batch, and the differentiable conformal predictor from Section \ref{sec:diff cp} is evaluated on $B_{\text{cal}}$ and $B_{\text{test}}$. After training, we apply the non-smooth CP procedure on the separate calibration and test sets, $\mathcal{S}_{\text{cal}}$ and $\mathcal{S}_{\text{test}}$. In this section, we describe in detail the loss function that is optimized during this process.

\subsection{\conftr{} with Inefficiency Regularization}\label{sec:ConfTr}
The objective of \conftr{} \cite{stutz2021learning} is to reduce the inefficiency, i.e., the size of the prediction set.
Ideally, each set should contain only the ground-truth label, resulting in an expected size of $1$; this is captured in a corresponding CP regularizer defined as:

\begin{equation}
    \mathcal{L}_{tr} = \frac{1}{|B_{test}|}\sum_{X\in B_{test}} \text{ReLU}\big(\sum_{k=1,-1}C_{\theta,k}(X;\tilde{\tau})-1\big).
\end{equation}
Recall that $C_{\theta,k}(X;\tilde{\tau})$ is considered as a soft assignment of class $k$ to the confidence set. However, minimizing the CP loss does not guarantee that the assigned label is correct. Therefore, a standard binary classification loss $\mathcal{L}_{c}$ is necessary as well. In our experiments, we adopt the STL-specific loss from \cite{li2024tlinet}. The total loss is a weighted sum of classification loss and CP regularization:
\begin{equation}\label{eq:standard loss}
    \mathcal{L} = \mathcal{L}_{c} + \lambda \mathcal{L}_{tr},
\end{equation}
where $\lambda\in\mathbb{R}$ adjusts the influence of the CP loss during training. The complete training procedure for \conftr{} is outlined in Section \ref{alg:standard}.

There are several limitations in \conftr’s loss design.
First, directly minimizing inefficiency is problematic. The ideal set size is one, yet the loss does not penalize sizes below one.\footnote{Intuitively, inefficiency can drop below one when CP knows that the result of the predictor is false, leading to an empty set~\cite{shafer2008tutorial}.
Such case is common in binary classification, especially at low confidence level $\alpha$.}
Second, the overall objective is a weighted sum, which requires careful tuning of the hyperparameter $\lambda$. Selecting $\lambda$ through cross-validation adds computational overhead and noticeably affects coverage and efficiency. These issues motivate our proposed $p$-value based loss used in \tlicp, introduced next.

\begin{algorithm}
\caption{\conftr{} with Inefficiency Regularization}
\label{alg:standard}
\begin{algorithmic}[1]
\Require $\alpha\in[0,1]$, $\lambda\in\mathbb{R}$.
\For{batch $B$}
    \State randomly split $B = B_{\text{cal}}\uplus B_{\text{test}}$.
    \For{$\{X_i,Y_i\}\in B_{\text{cal}}$}
    \State compute the nonconformity scores $\tilde{E}_{\theta}(X_i,Y_i)$.
    \EndFor
    \State $\tilde{\tau} = \text{SMOOTHCAl}(\{\pi_{\theta}(X_i),Y_i)\}_{i=1,...,n},\alpha)$.
    \For{$k\in\{1,-1\}$}
    \For{$\{X,Y\}\in B_{\text{test}}$}
    \State $C_{\theta,k}(X;\tilde{\tau}) = \text{SMOOTHPRED}(\pi_{\theta}(X),\tilde{\tau})$.
    \EndFor
    \EndFor
    \State $\mathcal{L} = \mathcal{L}_{c} + \lambda \mathcal{L}_{tr}$.
    \State update parameters $\theta$.
\EndFor
\end{algorithmic}
\end{algorithm}

\subsection{\tlicp{} with $p$-Value Optimization}
To address the limitations above, we propose a new loss function focusing on optimizing $p$-values instead of inefficiency. The objective is to maximize the $p$-value of the true label so it is included in the prediction set, while simultaneously minimize the $p$-value of false label so that it is excluded. The corresponding CP loss is
\begin{equation}\label{eq:p loss}
    \mathcal{L}_{cp} = \frac{1}{|B_{test}|}\sum_{(X,Y)\in B_{test}} (p_{\theta}(X;-1)Y-p_{\theta}(X;1)Y).
\end{equation}
When $Y=1$, the CP loss encourages $p_{\theta}(X;1)$ to increase and $p_{\theta}(X;-1)$ to decrease. Since larger $p$-values corresponds to smaller nonconformity scores, minimizing $\mathcal{L}_{cp}$ pushes robustness toward the correct side of the decision boundary, thereby reducing the MCR.

A notable property of this loss is its independence from the confidence level $\alpha$. In \conftr, $\alpha$ must be fixed during training to compute a threshold $\tau$, which ties optimization to that specific level and can halt progress once the predicted set reaches size one. Our $p$-value loss continues to encourage larger true-label $p$-values and smaller false-label $p$-values regardless of $\alpha$. Any desired $\alpha$ can then be selected post hoc with non-smooth CP, without retraining. Section \ref{alg:inductive} summarizes the full procedure.

\begin{remark}
If a specific $\alpha$ must be enforced during training, an alternative is
\begin{equation}
  \label{eq:p loss alpha}
  \begin{aligned}
    \mathcal{L}_{cp}^{(\alpha)} &= \frac{1}{|B_{test}|}\sum_{(X,Y)\in B_{test}} \big(\text{ReLU}(Y(\alpha-p_{\theta}(X;1)))\\
                        &+ \text{ReLU}(Y(p_{\theta}(X;-1)-\alpha)\big).
  \end{aligned}
\end{equation}
Note that for \emph{hard} examples where $p_{\theta}(X;Y)<\alpha$ and $p_{\theta}(X;-Y)>\alpha$, the corresponding terms in \protect\eqref{eq:p loss alpha} are the same as \eqref{eq:p loss}; however, Section \ref{eq:p loss alpha} does not explicitly penalize \emph{easy} examples.
\end{remark}

The total loss $\mathcal{L}$ is:
\begin{equation}
    \mathcal{L} = \mathcal{L}_{cp}.
\end{equation}
Compared to the loss in \eqref{eq:standard loss}, our loss function uses a single term to simultaneously promotes prediction accuracy and the size of confidence sets. This simplifies the training pipeline, reduces computational cost, and removes the need to tune the weighting hyperparameter $\lambda$.

% \rtronhl{This comes as a surprise, since we haven't discussed the multiclass case before. It is probably better to remove this.}{This loss can be generalized to multi-class classification problem as:}
% \begin{equation}
%     \mathcal{L}_{cp} = -\frac{1}{|B_{test}|}\sum_{(X,Y)\in B_{test}} \bigl((p_{\theta}(X;Y) - \sum_{k\neq Y} p_{\theta}(X;k) \bigr).
% \end{equation}

\begin{algorithm}
\caption{\tlicp{} with $p$-Value Optimization}
\label{alg:inductive}
\begin{algorithmic}[1]
\Require $\lambda\in\mathbb{R}$.
\For{batch $B$}
    \State randomly split $B = B_{cal}\uplus B_{test}$.
    \For{$\{X_i,Y_i\}\in B_{cal}$}
    \State compute the nonconformity scores $\tilde{E}_{\theta}(X_i,Y_i)$.
    \EndFor
    \For{$k\in\{1,-1\}$}
    \For{$\{X,Y\}\in B_{test}$}
    \State $p_{\theta}(X;k) = \text{DIFFP}($\\
    \quad \quad \quad \quad \quad$\{\tilde{E}_{\theta}(X,k)\},\{\tilde{E}_{\theta}(X_i,Y_i)\}_{i=1,...,n})$.
    \EndFor
    \EndFor
    \State $\mathcal{L} = \mathcal{L}_{cp}$.
    \State update parameters $\theta$.
\EndFor
\end{algorithmic}
\end{algorithm}

In the limit of small temperature and dispersion parameters ($T \rightarrow 0$) and given sufficiently large calibration sets $B_{\text{cal}}$, our differentiable conformal prediction method aims to achieve the specified coverage level $1-\alpha$ on the test set $B_{\text{test}}$. Empirically, we observe that practical implementations indeed yield coverage close to this theoretical limit. Importantly, these differentiable (smooth) approximations are employed exclusively during model training to facilitate gradient-based optimization. Once the model parameters have been learned, non-smooth conformal predictors are used to generate prediction sets, ensuring strict validity of the coverage guarantee.

%%%%%%%%%%%%%%%%%%%%%%%%%%%%%%%%%%%%%%%%%%%%%%%%%%%%%%%%%%%%%%%%%%%%%%%%%%%%%%%%
\section{CASE STUDIES}
In this section, we empirically evaluate \tlicp{} for STL inference and quantify its effect on the learned formula $\phi$. We apply \tlicp{} to infer STL formulas from trajectories that complete spatial-temporal tasks in the VIMA dataset~\cite{jiang2022vima}. Our implementation builds on the STL inference network TLINet~\cite{li2024tlinet}: TLINet models $\phi$ as a predictor, while \tlicp{} is embedded in the training process to differentiate through conformal prediction. We compare against two methods: TLINet without conformal training (Baseline) and TLINet with \conftr~\cite{stutz2021learning}.

\subsection{Dataset Setup}
To evaluate our conformal prediction framework for STL inference, we conduct experiments on the VIMA dataset, which features multimodal robot manipulation tasks with spatial-temporal requirements that can be described using STL formulas.

We instantiate two representative tasks in VIMA-BENCH:
\begin{enumerate}[label=Task \arabic*,leftmargin=0.2in,itemindent=0.25in]
  \item\label{task:reach}\textit{Place a block into the basket without violating the constraint}; this is a classical reach-avoid task.
  \item\label{task:sequence}\textit{First place block A, then place block B into the basket without violating the constraint}; this task is similar to \ref{task:reach}, but requires also temporal reasoning about the sequence of blocks.
\end{enumerate}
Each sample is a robot trajectory attempting to satisfy the task objective. A trajectory is labeled as positive if the objective is successfully achieved; otherwise, it is labeled as negative (e.g., the block misses the basket, a constraint is violated, or the required order is not followed). Following prior work, we use the end-effector $(x,y)$ coordinates as a trajectory for STL inference, discarding image and other modalities for simplicity and consistency. For each task, we collect 2000 trajectories and randomly split them into training set ($S_{\mathrm{train}}$, 60\%), calibration set ($S_{\mathrm{cal}}$, 20\%), and testing set ($S_{\mathrm{test}}$, 20\%). The trajectories in Task 1 are two-dimensional with a length of 20, while those in Task 2 are four-dimensional with two-dimensional positions for block A and block B, and have a length of 40. The environment setup and example trajectories are illustrated in Figures~\ref{fig:dataset}.

\begin{figure}[ht]
    \subfloat[VIMA simulated environment: the robot arm picks up a colored block from the table and places it into the designated basket region (red-white boundary) without violating the constraint (yellow bar).\label{fig:vima}]{\includegraphics[width=0.23\textwidth]{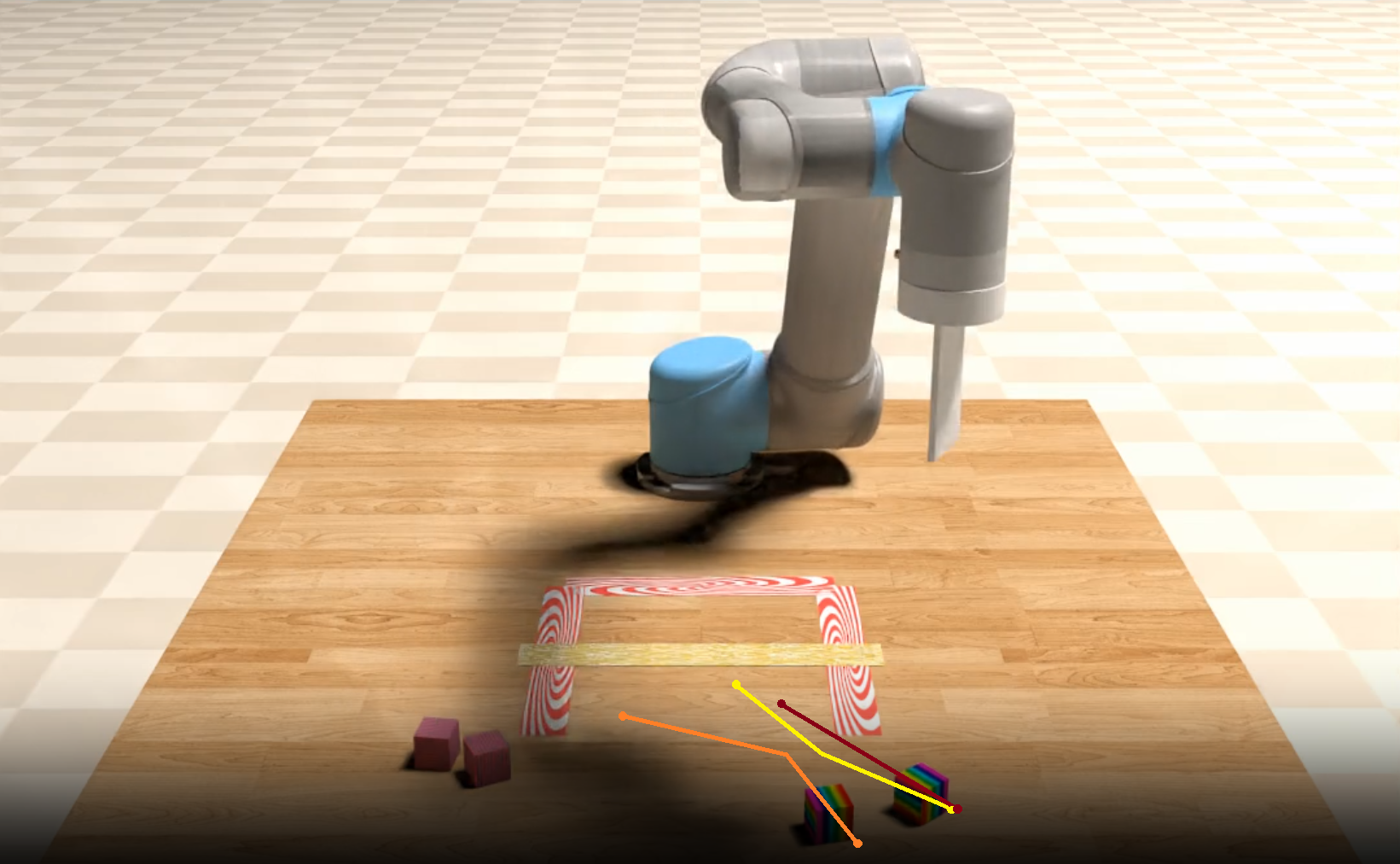}}
    \hfill
    \subfloat[Solid lines indicate successful placements (endpoints inside the basket), while dashed lines indicate failures. The green box is the valid region (basket and constraint).\label{fig:pos_trajs}]{\includegraphics[width=0.23\textwidth]{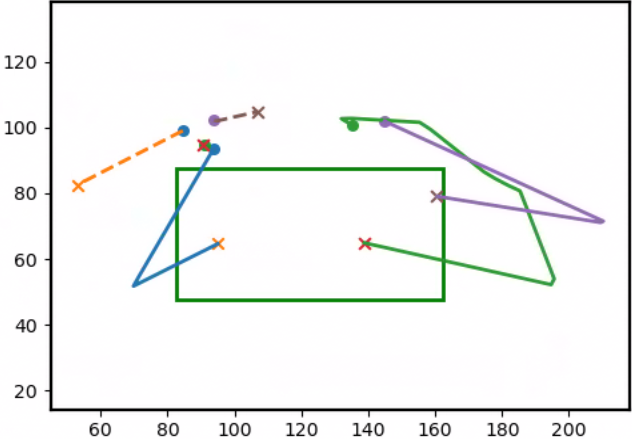}}
    \caption{VIMA dataset. (a) Task environment setup. (b) Representative trajectories demonstrating whether blocks are placed into the basket.}
    \label{fig:dataset}
\end{figure}

\subsection{Task 1: Placing a block into a basket}
In the first part, we evaluate STL inference on trajectories for a pick-and-place task: place a block into the basket without violating state constraints. We compare \tlicp{} with \conftr{} and the \emph{Baseline}. All methods use the nonconformity scores defined in Section \ref{sec:score}. \tlicp{} is independent of the confidence level $\alpha$, thus it is trained only once. \conftr{} requires retraining for each different value of $\alpha$; in our case we train five models with $\alpha\in\{0.1,0.05,0.01,0.005,0.001\}$. After training, prediction sets are computed with non-smooth CP as the confidence level $\alpha$ varies from $0.001$ to $0.1$ (in steps of $0.001$). We report the inefficiency and misclassification rate (MCR); the results are shown in Figure~\ref{fig:task_results} and Table~\ref{tab:task1_results}.

Figure~\ref{fig:task_results} plots inefficiency. From a CP perspective, a model is better if it maintains singleton sets down to a lower $\alpha$ (equivalently, has a lower $\alpha$ at which the curve first rises above 1 toward 2), implying size-1 prediction sets at a higher coverage. \tlicp{} attains an average set size of 1 for $\alpha$ as low as $0.005$, matching the lowest level reached by \conftr. Thus, \tlicp{} matches \conftr’s best performance without fixing $\alpha$ a priori. \conftr{} sometimes yields set sizes below 1 because its loss does not penalize sizes under 1, as discussed in Section~\ref{sec:ConfTr}. At $\alpha=0.1$ and $0.05$, \conftr’s inefficiency is worse than the Baseline, it is likely because once size 1 is easily achieved, the loss continues pushing set size below 1 instead of improving the classifier. While at the very small $\alpha=0.001$, \conftr{} again underperforms the Baseline, likely because the coverage target is too stringent for the model and degrades training. Overall, \conftr{} is sensitive to $\alpha$ and requires costly cross-validation, whereas \tlicp{} trains independently of $\alpha$, avoids this tuning, and still matches \conftr’s best performance.

Table~\ref{tab:task1_results} reports the MCRs, which are broadly similar across methods. Notably, at relatively larger $\alpha$ (0.1, 0.05), \conftr's MCRs are slightly higher than Baseline's. This is likely because the weighted loss in \eqref{eq:standard loss}: compared with the Baseline, which minimizes only classification error, \conftr{} allocates capacity to reducing inefficiency (pushing it below 1) rather than MCR, which can hurt accuracy. For small $\alpha$, this tradeoff is less problematic because achieving size 1 is difficult thus shrinking the set generally also reduces MCR.

The STL formula learned by \tlicp{} is
\begin{equation}
     \phi = G_{[17,19]} ((99.86<x<147.14) \land (59.64<y<89.97)),
\end{equation}
which means the block remains inside this bounding box during the final segment of the trajectory, $t\in[17,19]$. Figure~\ref{fig:predicates} visualizes this box alongside those learned by other methods. Methods that perform better on inefficiency (such as \tlicp{} and \conftr{} with $\alpha=0.005$) produce visibly tighter boxes around positive trajectories. Intuitively, tighter boxes more precisely capture the spatial information of successful executions. We also show in Figure~\ref{fig:strip} a jittered strip plot of the distributions of positive and negative samples in the test data across different methods. To facilitate comparison, we normalize the robustness values so that the margin lies between $-1$ and $1$. This normalization does not affect the results, as guaranteed by Lemma~\ref{lemma:invariant}. Although the means of the distributions for \tlicp{} and \conftr{} with $\alpha=0.005$ are closer than those of other methods, their distribution are better separated, i.e., they contain fewer samples inside the margin and fewer wrongly classified samples. Accordingly, \tlicp{} yields a more faithful and statistically supported specification of the task behavior.

%%%%%%%%%%%%%%%
\begin{table}[ht]
\centering
\subfloat[\label{tab:task1_results} Task 1]{
\scriptsize
\setlength\tabcolsep{4.7pt}
\begin{tabular}{l c c c c c c c}
\toprule
 & Baseline & \tlicp{} & \multicolumn{5}{c}{\conftr} \\
\cmidrule(lr){4-8}
 &  &  & $\alpha=0.1$ & $0.05$ & $0.01$ & $0.005$ & $0.001$ \\
\midrule
MCR & 0.0125 & 0.0050 & 0.0475 & 0.0450 & 0.0075 & 0.0125 & 0.0250 \\
\bottomrule
\end{tabular}
}

\subfloat[\label{tab:task2_results} Task 2]{
\centering
\scriptsize
\setlength\tabcolsep{4.7pt}
\begin{tabular}{l c c c c c c c}
\toprule
 & Baseline & \tlicp{} & \multicolumn{5}{c}{\conftr} \\
\cmidrule(lr){4-8}
 &  &  & $\alpha=0.1$ & $0.05$ & $0.01$ & $0.005$ & $0.001$ \\
\midrule
MCR & 0.0538 & 0.0218 & 0.0403 & 0.0403 & 0.0235 & 0.0319 & 0.0416 \\
\bottomrule
\end{tabular}
}
\caption{Average Misclassification Rate (MCR) on the Test Set.}
\label{tab:task_results}
\end{table}
%%%%%%%%%%%%%%%
\begin{figure}[ht]
    \centering
    \includegraphics[width=\columnwidth]{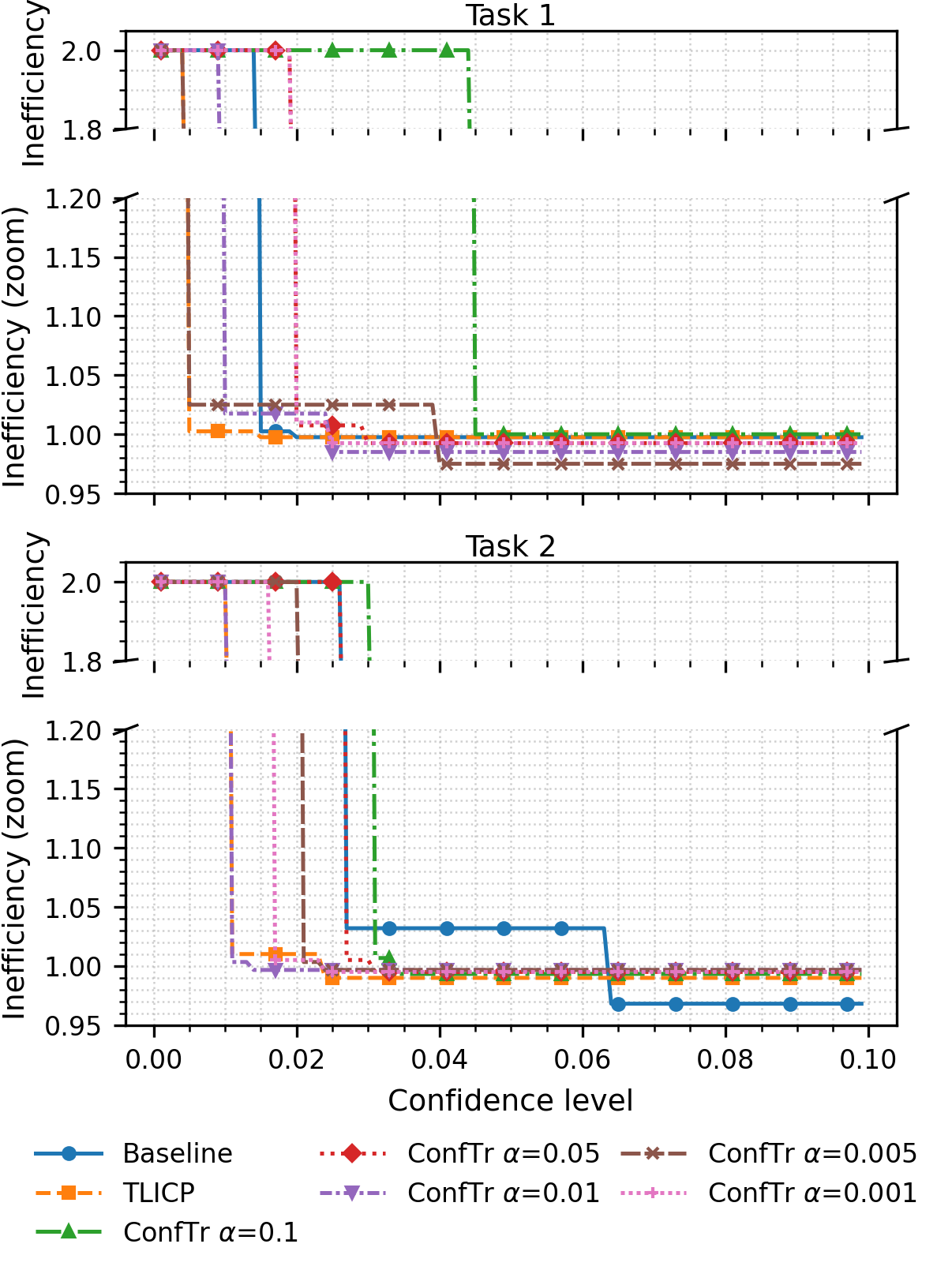}
    \caption{Confidence Level ($\alpha$) vs Average Inefficiency (Set Size) for Task 1 and Task 2.}
    \label{fig:task_results}
\end{figure}

\begin{figure}[ht]
    \centering
    \includegraphics[width=\columnwidth]{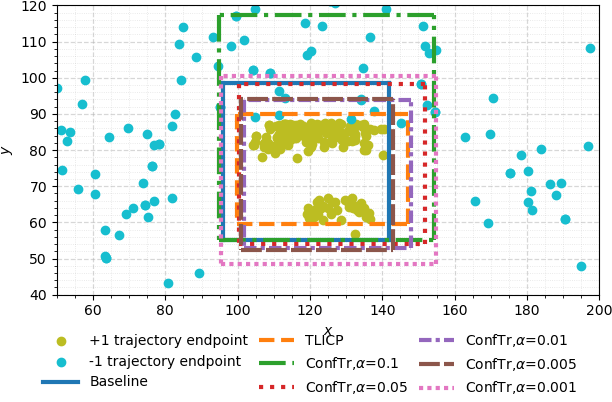}
    \caption{Specification regions inferred via STL from each method. Colored/linestyled boxes denote the learned axis-aligned predicates; points show test trajectory endpoints (positive and negative).}
    \label{fig:predicates}
\end{figure}

\begin{figure}[ht]
    \centering
    \includegraphics[width=0.95\columnwidth]{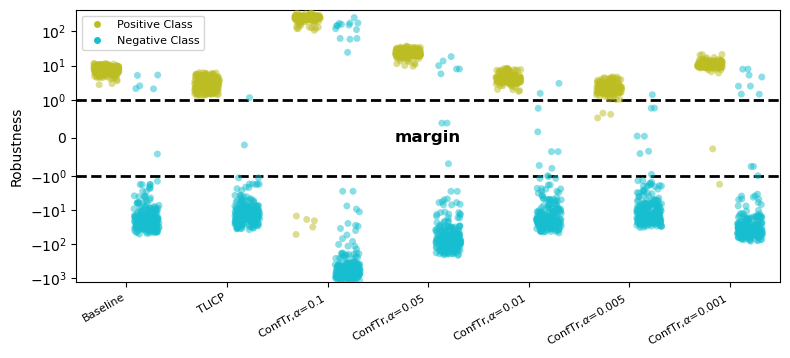}
    \caption{Jittered strip plot of the robustness distributions of samples across different methods. Each dot corresponds to an individual sample, with a small horizontal jitter added to reduce overlap and reveal density. The dashed horizontal lines at $y=\pm 1$ mark the margin region.}
    \label{fig:strip}
\end{figure}

%%%%%%%%%%%%%%%
\subsection{Task 2: Placing two blocks into a basket in order}
In the second part, we evaluate STL inference on a more complex task: placing block A and then block B into the basket without violating the constraint. We use the same setup as Task 1. Results are shown in Figure~\ref{fig:task_results} and Table~\ref{tab:task2_results}.

From Figure~\ref{fig:task_results} (Task 2), the smallest $\alpha$ at which average inefficiency stays at $1$ is $\alpha=0.01$ for both \tlicp{} and \conftr. It is reasonable since this task is more complicated than Task 1. However, \tlicp{} still matches \conftr's best achievable performance without tuning $\alpha$. Pushing $\alpha$ below $0.01$ expands the prediction sets and increases average inefficiency. Hence, selecting an appropriate $\alpha$ matters; simply decreasing $\alpha$ does not help and can make the sets unnecessarily conservative. Task 2 demands richer spatial-temporal information to classify trajectories, so the MCRs are generally higher than Task 1. The inefficiency is also more difficult to reduce. Consequently, the weighted loss tradeoff discussed earlier is less tricky here, and all \conftr{} models outperform the Baseline on MCR. The STL formula learned by \tlicp{} for Task 2 is
\begin{equation}
\phi = G_{[24,38]} (112.10<x_1<146.98 \land 63.74<y_1<88.83),
\end{equation}
which specifies that block A remains within the indicated bounding box during the final half of the trajectory ($t\in[24,38]$). Notably, the formula excludes predicates involving block B, suggesting that the behavior of block A alone is sufficient for accurate classification: the task requires block A to reach the basket in the first half of the trajectory and remain there thereafter. The conformal prediction procedure naturally favors this simplified yet sufficient specification.

%%%%%%%%%%%%%%%%%%%%%%%%%%%%%%%%%%%%%%%%%%%%%%%%%%%%%%%%%%%%%%%%%%%%%%%%%%%%%%%%
\section{CONCLUSIONS}
We presented \tlicp, a differentiable conformal-prediction framework for STL inference. Our method combines a robustness-based, margin-guided, unit-invariant nonconformity score, with a smooth CP module trained via a single-term $p$-value objective, eliminating the need to pre-specify $\alpha$ or tune loss weights. After training, exact CP is applied to obtain coverage guarantees. By jointly addressing interpretability and uncertainty quantification, \tlicp{} advances expressive, statistically rigorous temporal logic learning. Future work includes extending to multi-class STL inference and exploring conformalized quantile regression for real-valued robustness.

\bibliographystyle{IEEEtran}
\bibliography{STLCP}

\end{document}